\documentclass[sigconf]{acmart}

\AtBeginDocument{%
  \providecommand\BibTeX{{%
    \normalfont B\kern-0.5em{\scshape i\kern-0.25em b}\kern-0.8em\TeX}}}

\setcopyright{acmcopyright}
\copyrightyear{2018}
\acmYear{2018}
\acmDOI{10.1145/1122445.1122456}

\acmConference[Woodstock '18]{Woodstock '18: ACM Symposium on Neural
  Gaze Detection}{June 03--05, 2018}{Woodstock, NY}
\acmBooktitle{Woodstock '18: ACM Symposium on Neural Gaze Detection,
  June 03--05, 2018, Woodstock, NY}
\acmPrice{15.00}
\acmISBN{978-1-4503-XXXX-X/18/06}

\usepackage[utf8]{inputenc} 
\usepackage[T1]{fontenc}    
\usepackage{hyperref}       
\usepackage{url}            
\usepackage{booktabs}       
\usepackage{amsfonts}       
\usepackage{nicefrac}       
\usepackage{microtype}      
\usepackage{xcolor}         

\usepackage{graphicx}
\usepackage{subfigure}
\usepackage{amsmath}
\usepackage{amsthm}
\usepackage{mathtools}
\usepackage{wrapfig}
\usepackage{bbm}
\usepackage{multirow}
\usepackage{boldline}
\RequirePackage{algorithm}
\RequirePackage{algorithmic}

\newtheorem{theorem}{Theorem}
\newtheorem{corollary}[theorem]{Corollary}

\newtheorem{proposition}[theorem]{Proposition}

\newtheorem{definition}{Definition}[section]
\newtheorem{assumption}{Assumption}

\theoremstyle{remark}
\newtheorem{remark}{Remark}

\newtheorem*{psket*}{Proof sketch}

\newcommand{\pluseq}{\mathrel{{+}{=}}}

\copyrightyear{2022}
\acmYear{2022}
\setcopyright{acmcopyright}\acmConference[KDD '22]{Proceedings of the 28th ACM SIGKDD Conference on Knowledge Discovery and Data Mining}{August 14--18, 2022}{Washington, DC, USA}
\acmBooktitle{Proceedings of the 28th ACM SIGKDD Conference on Knowledge Discovery and Data Mining (KDD '22), August 14--18, 2022, Washington, DC, USA}
\acmPrice{15.00}
\acmDOI{10.1145/3534678.3539328}
\acmISBN{978-1-4503-9385-0/22/08}

\begin{document}

\title{Communication-Efficient Robust Federated Learning with Noisy Labels}

\author{Junyi Li}
\affiliation{
\institution{Electrical and Computer Engineering} 
  \institution{University of Pittsburgh}
  \country{United States}
  }
\email{junyili.ai@gmail.com}
\author{Jian Pei}
\affiliation{
\institution{School of Computing Science} 
  \institution{Simon Fraser University}
  \country{Canada}
  }
\email{jpei@cs.sfu.ca}
\author{Heng Huang}
\affiliation{
  \institution{Electrical and Computer Engineering} 
  \institution{University of Pittsburgh}
  \country{United States}
}
\email{henghuanghh@gmail.com}
\authornote{This work was partially supported by  NSF IIS 1845666, 1852606, 1838627, 1837956, 1956002, IIA 2040588.}

\renewcommand{\shortauthors}{Junyi Li, Jian Pei, Heng Huang}


\begin{abstract}
Federated learning (FL) is a promising privacy-preserving machine learning paradigm over distributed located data. In FL, the data is kept locally by each user. This protects the user privacy, but also makes the server difficult to verify data quality, especially if the data are correctly labeled. Training with corrupted labels is harmful to the federated learning task; however, little attention has been paid to FL in the case of label noise. In this paper, we focus on this problem and propose a learning-based reweighting approach to mitigate the effect of noisy labels in FL. More precisely, we tuned a weight for each training sample such that the learned model has optimal generalization performance over a validation set. More formally, the process can be formulated as a Federated Bilevel Optimization problem. Bilevel optimization problem is a type of optimization problem with two levels of entangled problems. The non-distributed bilevel problems have witnessed notable progress recently with new efficient algorithms. However, solving bilevel optimization problems under the Federated Learning setting is under-investigated. We identify that the high communication cost in hypergradient evaluation is the major bottleneck. So we propose \textit{Comm-FedBiO} to solve the general Federated Bilevel Optimization problems; more specifically, we propose two communication-efficient subroutines to estimate the hypergradient. Convergence analysis of the proposed algorithms is also provided. Finally, we apply the proposed algorithms to solve the noisy label problem. Our approach has shown superior performance on several real-world datasets compared to various baselines.
\end{abstract}



\maketitle

\section{Introduction}
In Federated Learning (FL)~\cite{mcmahan2017communication}, a set of clients jointly solve a machine learning problem under the coordination of a central server. To protect privacy, clients keep their own data locally and share model parameters periodically with each other. Several challenges of FL are widely studied in the literature, such as user privacy~\cite{nandakumar2019towards, mcmahan2017communication, wagh2019securenn}, communication cost~\cite{wen2017terngrad, lin2017deep, stich2018sparsified, karimireddy2019error, ivkin2019communication},  data heterogeneity~\cite{wang2019adaptive, haddadpour2019convergence,liang2019variance, karimireddy2019scaffold, bayoumi2020tighter} \emph{etc.}. However, a key challenge is ignored in the literature: \emph{the label quality of user data}. Data samples are manually annotated, and it is likely that the labels are incorrect. However, existing algorithms of FL \emph{e.g.} FedAvg~\cite{mcmahan2017communication} treat every sample equally; as a result, the learned models overfit the label noise, leading to bad generalization performance. It is challenging to develop an algorithm that is robust to the noise from the labels. Due to privacy concerns, user data are kept locally, so the server cannot verify the quality of the label of the user data. Recently, several intuitive approaches~\cite{chen2020focus, yang2020robust, tuor2021overcoming} based on the use of a clean validation set have been proposed in the literature. In this paper, we take a step forward and formally formulate the noisy label problem as a bilevel optimization problem; furthermore, we provide two efficient algorithms that have guaranteed convergence to solve the optimization problem.

The basic idea of our approach is to identify noisy label samples based on its contribution to training. More specifically, we measure the contribution through the Shapley value~\cite{shapley1951notes} of each sample. Suppose that we have the training dataset $\mathcal{D}$ and a sample $s \in \mathcal{D}$. Then for any subset $\mathcal{S} \subset \mathcal{D}/\{s\}$, we first train a model with $\mathcal{S}$ only and measure the generalization performance of the learned model, then train over $S\cup \{s\}$ and calculate the generalization performance again. The difference in generalization performance of the two models reflects the quality of the sample label. If a sample has a correct label, the model will have better generalization performance when the sample is included in the training; in contrast, a mislabeled sample harms the generalization performance. Then we define the Shapley value of any sample as the average of the generalization performance difference in all possible subsets $\mathcal{S}$. However, the Shapley value of a sample is NP-hard to compute. As an alternative, we define a weight for each sample and turn the problem into finding weights that lead to optimal generalization performance. With this reformulation, we need to solve a bilevel optimization problem. Bilevel optimization problems~\cite{willoughby1979solutions, solodov2007explicit, sabach2017first} involve two levels of problems: an inner problem and an outer problem.  Efficient gradient-based alternative update algorithms~\cite{ji2021lower, huang2021enhanced, li2021fully} have recently been proposed to solve non-distributed bilevel problems, but efficient algorithms designed for the FL setting have not yet been shown. In fact, the most challenging step is to evaluate the hypergradient (gradient \emph{w.r.t} the variable of the outer problem). In FL, hypergradient evaluation involves transferring the Hessian matrix, which leads to high communication cost. It is essential to develop a communication-efficient algorithm to evaluate the hypergradient and solve the Federated Bilevel Optimization problem efficiently.

More specifically, we propose two compression algorithms to reduce the communication cost for the hypergradient estimation: an iterative algorithm and a non-iterative algorithm. In the non-iterative algorithm, we compress the Hessian matrix directly and then solve a small linear equation to get the hypergradient. In the iterative algorithm, we formulate the hypergradient evaluation as solving a quadratic optimization problem and then run an iterative algorithm to solve this quadratic problem. To further save communication, we also compress the gradient of the quadratic objective function. Both the non-iterative and iterative algorithms effectively reduce the communication overhead of hypergradient evaluation. In general, the non-iterative algorithm requires communication cost polynomial to the stable rank of the Hessian matrix, and the iterative algorithm requires $O(log(d))$ ($d$ is the dimension of the model parameters). Finally, we apply the proposed algorithms to solve the noisy label problem on real-world datasets. Our algorithms have shown superior performance compared to various baselines. We highlight the contribution of this paper below.
\begin{enumerate}
\setlength{\itemsep}{-2pt}
    \item We study Federated Learning with noisy labels problems and propose a learning-based data cleaning procedure to identify mislabeled data. 
    \item We formalize the procedure as a Federated Bilevel Optimization problem. Furthermore, we propose two novel efficient algorithms based on compression, \emph{i.e.} the Iterative and Non-iterative algorithms. Both methods reduce the communication cost of the hypergradient evaluation from $O(d^2)$ to be sub-linear of $d$.
    \item We show that the proposed algorithms have a convergence rate of $O(\epsilon^{-2})$ and validate their efficacy by identifying mislabeled data in real-world datasets.
\end{enumerate}


\noindent \textbf{Notations.} $\nabla$ denotes the full gradient, $\nabla_{x}$ is the partial derivative for variable x, and higher-order derivatives follow similar rules. $||\cdot||$ is $\ell_2$-norm for vectors and the spectral norm for matrices. $||\cdot||_F$ represents the Frobenius norm. $AB$ denotes the multiplication of the matrix between the matrix $A$ and $B$. $\binom{n}{k}$ denotes the binomial coefficient. $[K]$ represents the sequence of integers from 1 to $K$.

\section{Related Works}
\label{sec:related-work}
\noindent\textbf{Federated Learning.} FL is a promising paradigm for performing machine learning tasks on distributed located data. Compared to traditional distributed learning in the data center, FL poses new challenges such as heterogeneity~\cite{karimireddy2019scaffold, sahu2018convergence, mohri2019agnostic, li2021ditto, huang2021compositional}, privacy~\cite{nandakumar2019towards, mcmahan2017communication, wagh2019securenn} and communication bottleneck~\cite{wen2017terngrad, lin2017deep, stich2018sparsified, karimireddy2019error, ivkin2019communication}. In addition, another challenge that receives little attention is the noisy data problem. Learning with noisy data, especially noisy labels, has been widely studied in a non-distributed setting~\cite{menon2015learning, patrini2017making, tanaka2018joint, shu2019meta, nishi2021augmentation, bao2019efficient, bao2020fast, bao2022distributed}.
However, since the server cannot see the clients' data and the communication is expensive between the server and clients, algorithms developed in the non-distributed setting can not be applied to the Federated Learning setting. Recently, several works~\cite{chen2020focus, yang2020robust, tuor2021overcoming} have focused on FL with noisy labels. In~\cite{chen2020focus}, authors propose FOCUS: The server defines a credibility score for each client based on the mutual cross-entropy of two losses: the loss of the global model evaluated on the local dataset and the loss of the local model evaluated on a clean validation set. The server then uses this score as the weight of each client during global averaging. In~\cite{tuor2021overcoming}, the server first trains a benchmark model, and then the clients use this model to exclude possibly corrupted data samples.

\noindent\textbf{Gradient Compression.} Gradient compression is widely used in FL to reduce communication costs. Existing compressors can be divided into quantization-based~\cite{wen2017terngrad, lin2017deep} and sparsification-based~\cite{stich2018sparsified, karimireddy2019error}. Quantization compressors give an unbiased estimate of gradients, but have a high variance~\cite{ivkin2019communication}. In contrast, sparsification methods generate biased gradients, but have high compression rate and good practical performance. The error feedback technique~\cite{karimireddy2019error} is combined with sparsification compressors to reduce compression bias. Sketch-based compression methods~\cite{ivkin2019communication, rothchild2020fetchsgd} are one type of sparsification compressor. Sketching methods~\cite{alon1999space} originate from the literature on streaming algorithms, \emph{e.g.} The Count-sketch~\cite{charikar2002finding} compressor was proposed to efficiently count heavy hitters in a data stream. 

\noindent\textbf{Bilevel Optimization.} Bilevel optimization~\cite{willoughby1979solutions} has gained more interest recently due to its application in many machine learning problems such as hyperparameter optimization~\cite{lorraine2018stochastic}, meta learning~\cite{zintgraf2019fast}, neural architecture search~\cite{liu2018darts} \emph{etc.} Various gradient-based methods are proposed to solve the bilevel optimization problem. Based on different approaches to the estimation of hypergradient, these methods are divided into two categories, \emph{i.e.} Approximate Implicit Differentiation (AID)~\cite{ghadimi2018approximation, ji2021lower, khanduri2021near, yang2021provably,huang2021enhanced, li2021fully, huang2021biadam} and Iterative Differentiation (ITD)~\cite{domke2012generic, maclaurin2015gradient, franceschi2017forward, pedregosa2016hyperparameter}. 
ITD methods first solve the lower level problem approximately and then calculate the hypergradient with backward (forward) automatic differentiation, while AID methods approximate the exact hypergradient~\cite{ji2020provably, ghadimi2018approximation, liao2018reviving, lorraine2018stochastic}.
In~\cite{grazzi2020iteration}, authors compare these two categories of methods in terms of their hyperiteration complexity. Finally, there are also works that utilize other strategies such as penalty methods~\cite{mehra2019penalty}, and also other formulations \emph{e.g.} the inner problem has multiple minimizers~\cite{li2020improved, sow2022constrained}. A recent work~\cite{li2022local} applied momentum-based acceleration to solve federated bilevel optimization problems.

\section{Preliminaries}
\noindent\textbf{Federated Learning.} A general formulation of Federated Learning problems is:
\begin{align}
     \underset{x \in \mathcal{X}}{\min}\ G(x) \coloneqq \frac{1}{N}\sum_{i=1}^{M} N_i g_i(x)
\label{eq:fedavg}
\end{align}
There are $M$ clients and one server. $N_i$ is the number of samples in the $i_{th}$ client and $N$ is the total number of samples. $g_i$ denotes the objective function on the $i_{th}$ client. To reduce communication cost, a common approach is to perform local sgd; in other words, the client performs multiple update steps with local data, and the model averaging operation occurs every few iterations. A widely used algorithm that uses this approach is FedAvg~\cite{mcmahan2017communication}. 

\noindent\textbf{Count Sketch.} The count-sketch technique ~\cite{charikar2002finding} was originally proposed to efficiently count heavy-hitters in a data stream. Later, it was applied in gradient compression: It is used to project a vector into a lower-dimensional space, while the large-magnitude elements can still be recovered. To compress a vector $g \in \mathbb{R}^d$, it maintains counters $r \times c$ denoted as $S$. Furthermore, it generates sign and bucket hashes $\{h_j^s, h_j^b\}_{j=1}^{r}$. In the compression stage,  for each element $g_i \in g$, it performs the operation $S[j, h_j^b(i)] \pluseq h_j^s[i] * g_i$ for $j \in [r]$. In the decompression stage, it recovers $g_i$ as $\text{median}(\{h_j^s[i]*S[j, h_j^b(i)]\}_{j=1}^r)$. To recover $\tau$ heavy-hitters (elements $g_i$ where $||g_i||^2 \ge \tau ||g||^2$) with probability at least $1 - \delta$, the count sketch needs $r \times c$ to be $O(\tau^{-1}\log(d/\delta))$. More details of the implementation are provided in~\cite{charikar2002finding}.

\noindent\textbf{Bilevel Optimization.} A bilevel optimization problem has the following form:
\begin{align}
    \underset{x \in \mathcal{X}}{\min}\ h(x) &\coloneqq F(x, y_x)\  \emph{s.t.}\ y_x = \underset{y\in \mathbb{R}^{d}}{\arg\min}\ G(x,y)
\label{eq:bi}
\end{align}
As shown in Eq.~\eqref{eq:bi}, a bilevel optimization problem includes two entangled optimization problems: the outer problem $F(x, y_x)$ and the inner problem $G(x, y)$. The outer problem relies on the minimizer $y_x$ of the inner problem. Eq.~\eqref{eq:bi} can be solved efficiently through gradient-based algorithms~\cite{ji2021lower, li2021fully}. There are two main categories of methods for hypergradient ( the gradient \emph{w.r.t} the outer variable $x$) evaluation: Approximate Implicit Differentiation (AID) and Iterative Differentiation (ITD). The ITD is based on automatic differentiation and stores intermediate states generated when we solve the inner problem. ITD methods are not suitable for the Federated Learning setting, where clients are stateless and cannot maintain historical inner states. In contrast, the AID approach is based on an explicit form of the hypergradient, as shown in Proposition~\ref{lemma:hyper-grad}:
\begin{proposition} (hypergradient)
\label{lemma:hyper-grad}
When $y_x$ is uniquely defined and $\nabla_{yy}^2 G(x,y_x)$ is invertible, the hypergradient has the following form:
\begin{equation}
\label{eq:formula_0}
\begin{split}
    \nabla h(x) =\ & \nabla_x F(x, y_x) -  \nabla_{xy}^2 G(x,y_x) v^{*}\\
\end{split}
\end{equation}
where $v^{*}$ is the solution of the following linear equation:
\begin{equation}
\label{eq:linear-eq}
    \nabla_{yy}^2 G(x,y_x)v^{*} = \nabla_y F(x, y_x) \\
\end{equation}
\end{proposition}
The proposition~\ref{lemma:hyper-grad} is based on the chain rule and the implicit function theorem. The proof of Proposition~\ref{lemma:hyper-grad} can be found in the bilevel optimization literature, such as~\cite{ghadimi2018approximation}. 

\section{Federated Learning with Noisy Labels}
We consider the Federated Learning setting as shown in Eq.~\eqref{eq:fedavg}, \emph{i.e.} a server and a set of clients. For ease of discussion, we assume that the total number of clients is $M$ and that each client has a private data set $\mathcal{D}_i = \{s^{i}_j, j\in [N_i]\}$, $i\in[M]$ where $|\mathcal{D}_i| = N_i$. $\mathcal{D}$ denotes the union of all client datasets: $\mathcal{D} = \bigcup_{i=1}^M \mathcal{D}_i$. The total number of samples $|\mathcal{D}| = N$ and $N = \sum_{i=1}^M N_i$ (for simplicity, we assume that there is no overlap between the client data sets).

In Federated Learning, the local dataset $\mathcal{D}_i$ is not shared with other clients or the server; this protects the user privacy, but also makes the server difficult to verify the quality of data samples. A data sample can be corrupted in various ways; we focus on the noisy label issue. Current Federated Learning models are very sensitive to the label noise in client datasets. Take the widely used FedAvg~\cite{mcmahan2017communication} as an example; the server simply performs a weighted average on the client models in the global averaging step. As a result, if one client model is affected by mislabeled data, the server model will also be affected. To eliminate the effect of these corrupted data samples on training, we can calculate the contribution of each sample. Based on the contribution, we remove samples that have little or even negative contributions. More specifically, we define the following metric of sample contribution based on the idea of Shapley value~\cite{shapley1951notes}:
\begin{align}
    \phi^{i}_j = \frac{1}{N} \sum_{\mathcal{S} \subset \mathcal{D} / \{s^i_j\}} \binom{N-1}{|\mathcal{S}|}^{-1} \left(\Phi\left(\mathcal {A}\left(\mathcal{S} \cup \{s^i_j\}\right)\right) - \Phi\left(\mathcal {A}\left(\mathcal{S}\right)\right)\right)
\label{eq:shap}
\end{align}
where $\mathcal{A}$ is a randomized algorithm (\emph{e.g.} FedAvg) that takes the dataset $\mathcal{S}$ as input and outputs a model. $\Phi$ is a metric of model gain, \emph{e.g.} negative population loss of the learned model, or negative empirical loss of the learned model in a validation set. In Eq.~\eqref{eq:shap}, for each $S \subset \mathcal{D} / \{s^i_j\}$, we calculate the marginal gain when $s^i_j$ is added to the training and then average over all such subsets $\mathcal{S}$. It is straightforward to find corrupted data if we can compute $\phi^i_j$, however, the evaluation of $\phi^i_j$ is NP-hard. As an alternative, we define the weight $\lambda^i_j \in [0, 1]$ for each sample and $\lambda^i_j$ should be positively correlated with the sample contribution $\phi^i_j$: large $\lambda$ represents a high contribution and small $\lambda$ means little contribution. In fact, finding sample weights that reflect the contribution of a data sample can be formulated as solving the following optimization problem:
\begin{align}
    &\underset{\lambda \in \Lambda}{\max}\ \Phi(\mathcal{A}(\mathcal{D}; \lambda))
\label{eq:shap_opt}
\end{align}
The above optimization problem can be interpreted as follows: The sample weights should be assigned so that the gain of the model $\Phi$ is maximized. It is then straightforward to see that only samples that have large contributions will be assigned with large weights. Next, we consider an instantiation of Eq.~\eqref{eq:shap_opt}. Suppose that we choose $\Phi$ as the negative empirical loss over a validation set $D_{val}$ at the server and $\mathcal{A}$ fits a model parameterized by $\omega$ over the data, then Eq.~\eqref{eq:shap_opt} can be written as:
\begin{align}
    &\underset{\lambda \in \Lambda}{\min}\ \ell(\omega_{\lambda}; \mathcal{D}_{val})\ \emph{s.t.}\ \omega_{\lambda} = \underset{\omega\in \mathbb{R}^d}{\arg\min}\ \frac{1}{N}\sum_{i=1}^{M}\sum_{j=1}^{N_i} \lambda^i_{j} \ell(\omega; s^i_j)
\label{eq:shap_example}
\end{align}
where $\ell$ is the loss function \emph{e.g.} the cross entropy loss. Eq.~\eqref{eq:shap_example} involves two entangled optimization problems: an outer problem and an inner problem, and $\omega_{\lambda}$ is the minimizer of the inner problem. This type of optimization problem is known as Bilevel Optimization Problems~\cite{willoughby1979solutions} as we introduce in the preliminary section. Following a similar notation as in Eq.~\eqref{eq:bi}, we write Eq.~\eqref{eq:shap_example} in a general form:
\begin{align}
    \underset{x \in \mathcal{X}}{\min}\ h(x) &\coloneqq F(x, y_x) \nonumber \\ \emph{s.t.}\ y_x &= \underset{y\in \mathbb{R}^{d}}{\arg\min}\ G(x,y) \coloneqq \frac{1}{N}\sum_{i=1}^{M} N_i g_i(x , y)
\label{eq:fed-bi}
\end{align}

\begin{algorithm}[t]
\caption{Communication-Efficient Federated Bilevel Optimization (\textbf{Comm-FedBiO})}
\label{alg:bifed-sketch-overall}
\begin{algorithmic}[1]
\STATE {\bfseries Input:} Learning rate $\eta, \gamma$, initial state ($x_0$, $y_0$), number of sampled clients $S$
\FOR{$k=0$ \textbf{to} $K-1$}
\STATE Sample $S$ clients and broadcast current model state ($x_k$, $y_k$);
\FOR{$m=1$ to $S$ clients in parallel}
\STATE Set $y^m_{0} = y_{k}$
\FOR{$t=1$ to $T$}
\STATE $y^m_{t+1} = y^m_{t} - \gamma \nabla g_m (x_k, y^m_t)$
\ENDFOR
\ENDFOR
\STATE $y_{k+1} = y_k + \frac{1}{\sum_{m=1}^{S} N_m}\sum_{m=1}^{S} N_m(y^m_T - y_k)$\\
// Two ways to estimate $\hat{\nabla} h(x_k)$ \\
\STATE Case 1: $\hat{\nabla} h(x_k) = \text{Iterative-approx}(x_k, y_{k+1})$\\
\STATE Case 2: $\hat{\nabla} h(x_k) = \text{Non-iterative-approx}(x_k, y_{k+1})$\\
\STATE $x_{k+1}~=~x_{k} - \eta\hat{\nabla} h(x_k)$
\ENDFOR
\end{algorithmic}
\end{algorithm}

Compared to Eq.~\eqref{eq:shap_example}, we set $\lambda$ as $x$, $\omega$ as $y$; $\ell(\omega_{\lambda}; \mathcal{D}_{val})$ as $F(x, y_x)$, and $1/N_i\sum_{j=1}^{N_i} \lambda^i_{j} \ell(\omega; s^i_j)$ as $g_i(x , y)$. In the remainder of this section, our discussion will be based on the general formulation~\eqref{eq:fed-bi}. We propose the algorithm \textbf{Comm-FedBiO} to solve Eq.~\eqref{eq:fed-bi} ( Algorithm~\ref{alg:bifed-sketch-overall}). Algorithm~\ref{alg:bifed-sketch-overall} follows the idea of alternative update of inner and outer variables in non-distributed bilevel optimization~\cite{ji2020provably, huang2021enhanced, li2021fully}, however, it has two key innovations which are our contributions. First, since the inner problem of Eq.~\eqref{eq:fed-bi} is a federated optimization problem, we perform local sgd steps to save the communication.
Next, we consider the communication constraints of federated learning in the hypergradient estimation. More specifically, we propose two communication-efficient hypergradient estimators, \emph{i.e.} the subroutine \textit{Non-iterative-approx} (line 11) and \textit{Iterative-approx} (line 12).

To see the high communication cost caused by hypergradient evaluation. We first write the hypergradient based on Proposition~\ref{lemma:hyper-grad},:
\begin{equation}
\begin{split}
\label{eq:formula}
    \nabla h(x) =& \nabla_x F(x, y_x) -  \nabla_{xy}^2 G(x,y_x) v^\ast \\  v^\ast =& \bigg(\sum_{i=1}^{M} \frac{N_i}{N}\nabla_{yy}^2g_i(x , y)\bigg)^{-1}\nabla_y F(x, y_x) \\
\end{split}
\end{equation}
where we use the explicit form of $v^\ast$ and replace $G(x,y)$ with the federated form in Eq.~\eqref{eq:fed-bi}. Eq.~\eqref{eq:formula} includes two steps: calculating $v^{\ast}$ and evaluating $\nabla h(x)$ based on $v^\ast$. For the second step, clients transfer $\nabla_{xy} g_i(x,y)v^\ast$ with communication cost $O(l)$ (the dimension of the outer variable $x$), we assume $l < d$ ($d$ is the dimension of $y$). This is reasonable in our noisy label application: $l$ is equal to the number of samples at a client and is very small in Federated Learning setting, while $d$ is the weight dimension, which can be very large. Therefore, we focus on the first step when considering the communication cost. The inverse of the Hessian matrix in Eq.~\eqref{eq:formula} can be approximated with various techniques such as the Neumann series expansion~\cite{ghadimi2018approximation} or conjugate gradient descent~\cite{ji2020provably}. However, clients must first exchange the Hessian matrix $\nabla_{yy}^2g_i(x , y)$. This leads to a communication cost on the order of $O(d^2)$. In fact, it is not necessary to transfer the full Hessian matrix, and we can reduce the cost through compression. More specifically, we can exploit the sparse structure of the related properties; \emph{e.g.} The Hessian matrix has only a few dominant singular values in practice. In Sections 4.1 and 4.2, we propose two communication-efficient estimators of hypergradient $\nabla h(x_k)$ based on this idea. In general, our estimators can be evaluated with the communication cost sublinear to the parameter dimension $d$.

\subsection{hypergradient Estimation with iterative algorithm} 
In this section, we introduce an iterative hypergradient estimator. Instead of performing the expensive matrix inversion as in Eq.~\eqref{eq:formula}, we solve the following quadratic optimization problem: 
\begin{equation}
\begin{split}
\label{eq:approx-lr2}
\underset{v}{\min}\ q(v) &\coloneqq \frac{1}{2}v^T\nabla_{yy}^2 G(x,y_x)v - v^T\nabla_y F(x, y_x)
\end{split}
\end{equation}

\begin{algorithm}[tb]
   \caption{Iterative Approximation of hypergradient (\textbf{Iterative-approx})}
   \label{alg:bifed-sketch2}
\begin{algorithmic}[1]
    \STATE {\bfseries Input:} State $(x, y)$, initial value $v_0$, learning rate $\alpha$, number of sampled clients $S$
    \STATE The server evaluates $\nabla_y F(x, y)$  and  samples $S$ clients uniformly and broadcasts state $(x, y)$ to each client;
    \FOR{$i = 0$ to $I-1$}
    \FOR{$m = 0$ to $S$ in parallel}
    \STATE Each client makes Hessian-vector product queries to compute $\nabla_{yy}^2 g_mv^{i}$, then send its sketch $S_{g_m}^i$ to the server;
    \ENDFOR
    \STATE {\bfseries Server:} $S_G^i = \frac{1}{\sum_{m=1}^S N_m}\sum_{m=1}^{S} N_m * S_{g_m}^i$
    \STATE {\bfseries Server:} $\Delta = U(\alpha S_G^i + S(e^{i}))$
    \STATE {\bfseries Server:} $v^{i+1} = v^{i} - (\Delta - \alpha\nabla_y F(x,y))$
    \STATE {\bfseries Server:} $S(e^{i+1}) = \alpha S_G^i + S(e^{i}) - S(\Delta)$
    \ENDFOR
    \STATE {\bfseries Output:}  $\hat{\nabla} h(x) = \nabla_x F - \nabla_{xy}^2 G v^I$ 
\end{algorithmic}
\end{algorithm}

The equivalence is observed by noticing that: \[\nabla q(v) = \nabla_{yy}^2 G(x,y_x)v - \nabla_y F(x, y_x)\] If $q(v)$ is strongly convex ($\nabla_{yy}^2 G(x,y_x)$ is positive definite), the unique minimizer of the quadratic function $q(v)$ is exactly $v^{*}$ as shown in Eq.~(\ref{eq:formula}). Eq.~(\ref{eq:approx-lr2}) is a simple positive definite quadratic optimization problem and can be solved with various iterative gradient-based algorithms. To further reduce communication cost, we compress the gradient $\nabla q(v)$. More specifically, $\nabla q(v)$ can be expressed as follows in terms of $g_i$:
\begin{align}
    \nabla q(v) = \frac{1}{N} \sum_{i=1}^M N_i\nabla_{yy}^2 g_i(x,y_x)v - \nabla_y F(x, y_x)
\label{eq:gradient}
\end{align}
Therefore, clients must exchange the Hessian vector product to evaluate $\nabla q(v)$. This operation has a communication cost $O(d)$. This cost is considerable when we evaluate $\nabla q(v)$ multiple times to optimize Eq.\eqref{eq:approx-lr2}. Therefore, we exploit compression to further reduce communication cost; \emph{i.e.} clients only communicate the compressed Hessian vector product. Various compressors can be used for compression. In our paper, we consider the local Topk compressor and the Count Sketch compressor~\cite{charikar2002finding} in our paper. The local Top-k compressor is simple to implement, but it cannot recover the global Top-k coordinates, while the count-sketch is more complicated to implement, but it can recover the global Top-k coordinates under certain conditions. In general, gradient compression includes two phases: compression at clients and decompression at the server. In the first phase, clients compress the gradient to a lower dimension with the compressor $S(\cdot)$, then transfer the compressed gradient to the server; In the second phase, the server aggregates the compressed gradients received from clients and decompresses them to recover an approximation of the original gradients. We denote the decompression operator as $U(\cdot)$. Then the update step of a gradient descent method with compression is as follows:
\begin{equation}
\label{eq:update}
\begin{split}
v^{i+1} &= v^{i} - C(\alpha \nabla q(v^{i}) + e^{i}),\\
e^{i+1} &= \alpha \nabla q(v^{i}) + e^{i} - C(\alpha \nabla q(v^{i}) + e^{i})  
\end{split}
\end{equation}
where $C(\cdot) \coloneqq U(S(\cdot))$ and $\alpha$ is the learning rate. Notice that we add an error accumulation term $e^{i}$. As shown by the update rule of $e^{i}$, it accumulates information that cannot be transferred due to compression and reintroduces information later in the iteration, this type of error feedback trick compensates for the compression error and is crucial for convergence. The update rule in Eq.~\eqref{eq:update} has communication cost sub-linear \emph{w.r.t} parameter dimension $d$ with either the Top-k compressor or the Count-sketch compressor. Furthermore, the iteration complexity of iterative algorithms is independent of the problem dimension, the overall communication cost of evaluating $v^\ast$ is still sublinear~\emph{w.r.t} the dimension $d$. This is a great reduction compared to the $O(d^2)$ complexity when we transfer the Hessian directly as in Eq.~\eqref{eq:formula}. We term this hypergradient approximation approach the iterative algorithm, and the pseudocode is shown in Algorithm~\ref{alg:bifed-sketch2}. Note that the server can evaluate $\nabla_y F(x,y)$, so we do not need to compress it, and we also omit the step of getting $\nabla_{xy}^2 Gv^I$ from the clients.

\subsection{hypergradient Estimation with Non-iterative algorithm} 
In this section, we propose an efficient algorithm so that we can estimate $v^\ast$ by solving the linear equation Eq.~(\ref{eq:linear-eq}) directly. However, instead of transferring $\nabla_{yy}^2 g_i(x,y_x)$, we transfer their sketch. More precisely, we solve the following linear equation:
\begin{equation}
\label{eq:approx-lr1}
    \begin{split}
        S_{2}\nabla_{yy}^2 G(x,y_x)S_{1}^T \omega = S_{2}\nabla_y F(x, y_x)\\
    \end{split}
\end{equation}
where $\hat{\omega} \in \mathbb{R}^{r_1}$ denotes the solution of Eq.~\eqref{eq:approx-lr1}. $S_{1} \in \mathbb{R}^{r_1 \times d}$ and $S_2 \in \mathbb{R}^{r_2 \times d}$ are two random matrices. Then an approximation of the hypergradient $\nabla h(x)$ is:
\begin{equation}
\label{eq:approx-lr3}
   \hat{\nabla} h(x) = \ \nabla_x F(x, y_x) -  \nabla_{xy}^2 G(x,y_x)S_{1}^T\hat{\omega}
\end{equation}
To solve Eq.~(\ref{eq:approx-lr1}), clients first transfer $S_{2} \nabla_{yy}^2 g_i(x,y_x)S_{1}^T$ to the server with communication cost $O(r_1r_2)$, then the server solves the linear system~(\ref{eq:approx-lr1}) locally. The server then transfers $\hat{\omega}$ to the clients and the clients transfer $\nabla_{xy} g_i(x, y_x)S_1^T\hat{\omega}$ back to the server, the server evaluates Eq.~(\ref{eq:approx-lr3}) to get $\hat{\nabla} h(x)$. The total communication cost is $O(r_1r_2)$ (we assume that the dimension of the outer variable $x$ is small).

We require $S_1$ and $S_2$ to have the following two properties: the approximation error $||\hat{\nabla} h(x) - \nabla h(x)||$ is small and the communication cost is much lower than $O(d^2)$, \emph{i.e.} $r_1r_2 \ll d^2$. We choose $S_1$ and $S_2$ as the following sketch matrices:
\theoremstyle{definition}
\begin{definition}
A distribution $\mathcal{D}$ on the matrices $S \in \mathbb{R}^{r\times n}$ is said to generate a $(\epsilon, \delta)$-sketch matrix for a pair of matrices $A$, $B$ with $n$ rows if:
\begin{equation*}
    \underset{S\sim\mathcal{D}}{\Pr} [||A^TS^TSB - A^TB|| > \epsilon||A||_F||B||_F] \le \delta
\end{equation*}
\end{definition}

\begin{corollary}
\label{cor:sub-embed}
An $(\epsilon/l, \delta)$ sketch matrix $S$ is a subspace embedding matrix for the column space of $A \in R^{n\times l}$. \emph{i.e.} for all $x \in \mathbb{R}^l$ \emph{w.p.} at least $1 - \delta$:
\begin{equation*}
    ||SAx||_2^2 \in [(1 -\epsilon)||Ax||_2^2, (1 +\epsilon)||Ax||_2^2]
\end{equation*}
\end{corollary}

\begin{corollary}
\label{lemma:sketch-mat}
For any $\epsilon, \delta \in (0, 1/2)$, let $S \in \mathbb{R}^{r\times n}$ be a random matrix with $r > 18/(\epsilon^2\delta)$. Furthermore, suppose that $\sigma \in \mathbb{R}^n$ is a random sequence where $\sigma(i)$ is randomly chosen from $\{-1,1\}$ and $h \in \mathbb{R}^n$ is another random sequence where $h(i)$ is randomly chosen from $[r]$. Suppose that we set $S[h(i), i] = \sigma(i)$, for $i \in [n]$ and 0 for other elements; then S is a $(\epsilon, \delta)$ sketch matrix.
\end{corollary}

Approximately, the sketch matrices S  are `invariant' over matrix multiplication ($\langle SA, SB\rangle \approx AB$). An important property of a ($\epsilon$, $\delta$)-sketch matrix is the subspace embedding property in Corollary~\ref{cor:sub-embed}: The norm of the vectors in the column space of $A$ is kept roughly after being projected by $S$.
Many distributions generate sketch matrices, such as \textit{sparse embedding matrice}~\cite{woodruff2014sketching}. We show one way to generate a sparse embedding matrix in Corollary~\ref{lemma:sketch-mat}. This corollary shows that we need to choose O($\epsilon^{-2}$) number of rows for a sparse embedding matrix to be a $(\epsilon, \delta)$ matrix. Finally, since we directly solve a (sketched) linear equation without using any iterative optimization algorithms, we term this hypergradient estimator as a non-iterative approximation algorithm.
The pseudocode summarizing this method is shown in Algorithm~\ref{alg:bifed-sketch}. We omit the subscript of iterates when it is clear from the context. Note that the server sends the random seed to ensure that all clients generate the same sketch matrices $S_1$ and $S_2$.

\begin{algorithm}[tb]
   \caption{Non-iterative approximation of hypergradient (\textbf{Non-iterative-approx})}
   \label{alg:bifed-sketch}
\begin{algorithmic}[1]
   \STATE {\bfseries Input:} State $(x, y)$, random seeds $\tau_1$, $\tau_2$, number of rows $r_1$, $r_2$, number of sampled clients $S$
   \STATE {\bfseries Server:} Sample $S$ clients uniformly and broadcast model state $ (x, y)$ to each sampled client
   \FOR{$m = 1$ to $S$ in parallel}
   \STATE Each client generates $S_1, S_{2}$ with random seeds $\tau_1$, $\tau_2$, compute $  S_{2}\nabla_{yy}^2 g_jS_1^T$ and $ \nabla_{xy}^2 g_jS_1^T$ with Hessian-vector product queries
   \ENDFOR
   \STATE {\bfseries Server:} Collect and average sketches from clients to get $S_{2}\nabla_{yy}^2 GS_1^T$ and $\nabla_{xy}^2 GS_1^T$ and solves Eq.~(\ref{eq:approx-lr1}) with linear regression to get $\hat{\omega}$
   \STATE {\bfseries Output:} $\hat{\nabla} h(x) = \nabla_x F - \nabla_{xy}^2 GS_1^T\hat{\omega}$
\end{algorithmic}
\end{algorithm}

\section{Convergence Analysis}
In this section, we analyze the convergence property of Algorithm~\ref{alg:bifed-sketch-overall}. We first state some mild assumptions needed in our analysis, then we analyze the approximation error of the two hypergradient estimation algorithms, \emph{i.e.} the iterative algorithm and the non-iterative algorithm. Finally, we provide the convergence guarantee of Algorithm~\ref{alg:bifed-sketch-overall}.

\subsection{Some Mild Assumptions} 
We first state some assumptions about the outer and inner functions as follows:
\begin{assumption}\label{basic_assumption}
The function $F$ and $G$ has the following properties:
\begin{itemize}
\item[a)] $F(x,y)$ is possibly non-convex, $\nabla_x F(x,y)$ and $\nabla_y F(x,y)$ are Lipschitz continuous with constant $L_F$
\item[b)] $\|\nabla_x F(x, y)\|$ and $\|\nabla_y F(x, y)\|$ are upper bounded by some constant $C_{F}$
\item[c)] $G(x,y)$ is continuously twice differentiable, and $\mu_G$-strongly convex \emph{w.r.t} $y$ for any given x
\item[d)] $\nabla_y G(x,y)$ is Lipschitz continuous with constant $L_{G}$
\item[e)] $\|\nabla_{xy}^2 G(x, y)\|$ is upper bounded by some constant $C_{G_{xy}}$
\end{itemize}
\end{assumption}

\begin{assumption}\label{higher-order_assumption}
$\nabla_{xy}^2 G(x,y)$ and $\nabla_{yy}^2 G(x,y)$ are Lipschitz continuous with constants $L_{G_{xy}}$ and $L_{G_{yy}}$, respectively.
\end{assumption}

In Assumptions~\ref{basic_assumption} and~\ref{higher-order_assumption}, we make assumptions about the function $G$, it is also possible to make stronger assumptions about the local functions $g_i$. Furthermore, these assumptions are used in the bilevel optimization literature~\cite{ghadimi2018approximation, ji2020provably}, especially, we require higher-order smoothness in Assumption~\ref{higher-order_assumption} as bilevel optimization is involved with the second-order information. The next two assumptions are needed when we analyze the approximation property of the two hypergradient estimation algorithms:
\begin{assumption}\label{sketch2_assumption}
For a constant $0 < \tau < 1$ and a vector $g \in \mathbb{R}^d$. If $\exists\ i$, such that $(g_{i})^2 \ge \tau||g||^2$, then $g$ has $\tau$-heavy hitters.
\end{assumption}

\begin{assumption}\label{sketch1_assumption}
The stable rank of $\nabla_{yy}^2 G(x,y_x)$ is bounded by $r_s$, \emph{i.e.}  $\sum_{i=1}^{d} \sigma_i^2 \le  r_s \sigma_{max}^2$, where ($\sigma_{max}$) $\sigma_i$ denotes the (max) singular values of the Hessian matrix.
\end{assumption}
The heavy-hitter assumption~\ref{sketch2_assumption} is commonly used in the literature to show the convergence of gradient compression algorithms. To bound the approximation error of our iterative hypergradient estimation error, we assume $\nabla q(v)$ defined in Eq.~\eqref{eq:gradient} to satisfy this assumption. Assumption~\ref{sketch1_assumption} requires the Hessian matrix to have several dominant singular values, which describes the sparsity of the Hessian matrix. We assume Assumption~\ref{sketch1_assumption} holds when we analyze the approximation error of the non-iterative hypergradient estimation algorithm.

\subsection{Approximation Error of the iterative algorithm}
In this section, we show the approximation error of the iterative algorithm. Suppose that we choose count-sketch as the compressor, we have the following theorem:
\begin{theorem}
\label{theo:sketch2}
Assume Assumptions~\ref{basic_assumption} and ~\ref{sketch2_assumption} hold. In Algorithm~\ref{alg:bifed-sketch2}, set the learning rate $\alpha = \frac{8}{\mu_{G}(i+a)}$ with $a>\max\left(1, \frac{2-\tau}{\tau}(\sqrt{\frac{2}{2-\tau}}+ 1)\right)$ as a shift constant. If the compressed gradient has dimension $O(\frac{log(dI/\delta)}{\tau})$, then with probability $1 - \delta$ we have:
\begin{equation*}
  E[||v^I - v^*||^2] \le \frac{C_1}{I^3} +\frac{C_2}{I^2} +  \frac{C_3(I+2a)}{I^2}
\end{equation*}
where $C_1$, $C_2$, and $C_3$ are constants.
\end{theorem}

\begin{remark}
The proof is included in Appendix A. As shown by Theorem~\ref{theo:sketch2}, the approximation error of Algorithm~\ref{alg:bifed-sketch2} is of the order of $O(1/I)$, and the constants encompass compression errors. Finally, the communication cost is of the order of $O(log(d))$, which is sublinear \emph{w.r.t} of dimension $d$.
\end{remark}

\subsection{Approximation Error of the Non-iterative algorithm}
In this section, we show the approximation error of the non-iterative algorithm. More precisely, we have Theorem~\ref{theo:sketch1}:
\begin{theorem}
\label{theo:sketch1}
For any given $\epsilon, \delta \in (0, 1/2)$, if $S_1 \in \mathbb{R}^{r_1 \times d}$ is a $(\lambda_1\epsilon, \delta/2)$ sketch matrix and $S_{2} \in \mathbb{R}^{r_2 \times d}$ is a $(\lambda_2\epsilon, \delta/2)$ sketch matrix. Under Assumptions~\ref{basic_assumption} and ~\ref{sketch1_assumption}, with probability at least $1 - \delta$, we have the following:
\begin{equation*}
\begin{split}
||\hat{\nabla} h(x) - \nabla h(x)|| \le \epsilon ||v^{*}||
\end{split}
\end{equation*}
where
$\lambda_1 = \frac{5\mu_{G}}{7\sqrt{r_s}C_{G_{xy}}L_{G}}, \lambda_2 = \frac{1}{3(r_1+1)}$ are constants.
\end{theorem}

\begin{psket*}
The main step is to bound $||S_1^T  \hat{\omega} - v^{*}||$, then the conclusion follows from the definition of the hypergradient. To bound $||S_1^T  \hat{\omega} - v^{*}||$, we use the approximation matrix multiplication property of $S_1$ and the subspace embedding property of $S_2$ to have:
$||S_1^T  \hat{\omega} - v^{*}||_2 \le C ||v^{*}||_F||\nabla_{yy}^2 G(x,y_x)||_F$, where C is some constant. The last step is to use the stable rank and the smoothness assumption to bound $||\nabla_{yy}^2 G(x,y_x)||_F$. The full proof is included in Appendix B.
\end{psket*}

\begin{remark}
Based on Corollary~\ref{lemma:sketch-mat}, we have an $(\epsilon, \delta)$ sketch matrix that has $r~=~ O(\epsilon^{-2})$ rows. Combining with Theorem~\ref{theo:sketch1}, we have $r_1=O(r_{s})$ and $r_2 = O(r_{s}^2)$. So, to reach the approximation error $\epsilon$, the number of rows of the sketch matrices is $O(r_{s}^2)$. This shows that the stable rank (the number of dominant singular values) correlates with the number of dimensions to be maintained after compression.
\end{remark}

\subsection{Convergence of the Comm-FedBiO algorithm}
In this section, we study the convergence property of the proposed communication efficient federated bilevel optimization (\textbf{Comm-FedBiO}) algorithm. First, $h(x)$ is smooth based on Assumptions~\ref{basic_assumption} and~\ref{higher-order_assumption}, as stated in the following proposition:
\begin{proposition}
\label{lemma:smooth}
Under Assumption~\ref{basic_assumption} and \ref{higher-order_assumption}, $\nabla h(x)$ is Lipschitz continuous with constant $L_h$, \emph{i.e.} 
\[||\nabla h(x_1) - \nabla h(x_2)|| \le L_h||x_1 - x_2||\] where $\nabla h(x)$ is the hypergradient and is defined in Proposition~\ref{lemma:hyper-grad}.
\end{proposition}
The proof of Proposition~\ref{lemma:smooth} can be found in the Lemma~2.2 of~\cite{ghadimi2018approximation}. We are ready to prove the convergence of Algorithm~\ref{alg:bifed-sketch-overall} in the following theorem. In this simplified version, we ignore the exact constants. A full version of the theorem is included in Appendix C.
\begin{theorem}
\label{theo:overall}
Under Assumption~\ref{basic_assumption} and \ref{higher-order_assumption}, if we choose the learning rate $\eta = \frac{1}{2L_h\sqrt{K+1}}$ in Algorithm~\ref{alg:bifed-sketch-overall},
\begin{itemize}
\item[a)] Suppose that $\{x_k\}_{k\ge0}$ is generated from the iterative Algorithm~\ref{alg:bifed-sketch2}. Under Assumption~\ref{sketch2_assumption}, for $I = O(\sqrt{K})$, we have the following:
\begin{equation*}
\begin{split}
    E[||\nabla h(x_k)||^2] \le&  \frac{C_1}{K^{3/2}} + \frac{C_2}{K} + \frac{C_3}{\sqrt{K}}
\end{split}
\end{equation*}
where $C_1$, $C_2$, $C_3$ are some constants
\item[b)] Suppose $\{x_k\}_{k\ge0}$ are generated from the non-iterative Algorithm~\ref{alg:bifed-sketch}. Under Assumption~\ref{sketch1_assumption}, for $\epsilon = O(K^{-1/4})$, it holds:
\begin{equation*}
    E[||\nabla h(x_k)||^2] \le \frac{C}{\sqrt{K}}
\end{equation*}
where $C$ is some constant.
\end{itemize}
\end{theorem}

\begin{psket*}
Firstly, by the smoothness of $h(x)$, we can upper-bound $h(x_{k+1})$ as:
\begin{equation*}
\begin{split}
    h(x_{k+1}) \le& h(x_k) - \eta(\frac{1}{2} - \eta L_h) ||\nabla h(x_k)||^2 \\
    &+ \eta (\frac{1}{2} + \eta L_h)||\hat{\nabla} h(x_k) - \nabla h(x_k)||^2\\
\end{split}
\end{equation*}
Next, we need to bound the error in the third term, where we can utilize the bound provided in Theorem~\ref{theo:sketch2} and Theorem~\ref{theo:sketch1}. Finally, we find a suitable averaging scheme to obtain the bound for $E[||\nabla h(x_k)||^2]$.
\end{psket*}

\begin{remark}
The convergence rate for the nonconvex-strongly-convex bilevel problem without using variance reduction technique is $O(1/\sqrt{K})$~\cite{ghadimi2018approximation}, thus both estimation algorithms achieve the same convergence rate as in the non-distributed setting. For the non-iterative algorithm, we need to scale $\epsilon = O(K^{-1/4})$, while the iterative algorithm instead scales the number of iterations as $O(\sqrt{K})$ at each hyper-iteration. 
Comparing these two methods: The iterative algorithm has to perform multiple rounds of communication, but it distributes the computation burden over multiple communication clients (multiple clients by sampling different clients at each step) and requires one Hessian vector product per round. But if the communication is very expensive, we could instead use the non-iterative algorithm, which requires one round of communication.
\end{remark}

\begin{figure}
\begin{center}
\includegraphics[width=0.48\columnwidth]{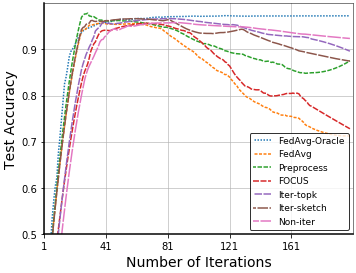}
\includegraphics[width=0.48\columnwidth]{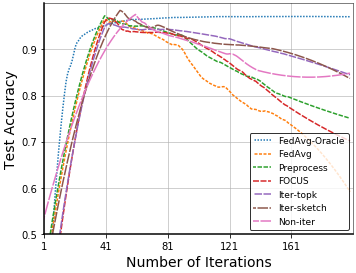}
\includegraphics[width=0.48\columnwidth]{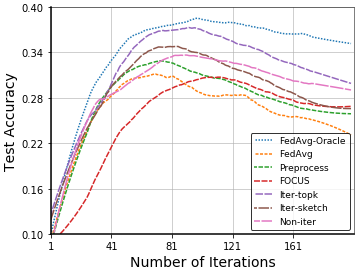}
\includegraphics[width=0.48\columnwidth]{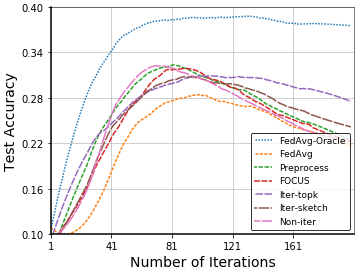}
\includegraphics[width=0.48\columnwidth]{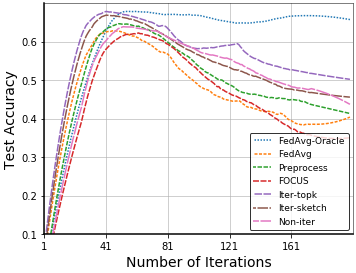}
\includegraphics[width=0.48\columnwidth]{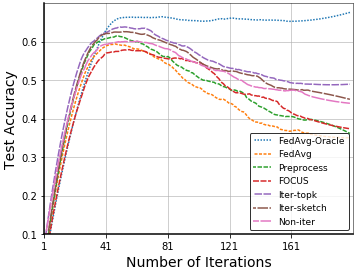}
\end{center}
\caption{Test accuracy plots for our Comm-FedBiO (three variants: Iter-topK, Iter-sketch and Non-iter) and other baselines. The plots show the results for the MNIST dataset, the CIFAR-10 dataset, and the FEMNIST dataset from top to bottom. The plots in the left column show the i.i.d. case, and plots in the right column show the non-i.i.d. case. The compression rate of our algorithms is 20$\times$ in terms of the parameter dimension $d$.}
\label{fig:data-clean}
\vspace{-0.3in}
\end{figure}

\vspace{-0.1in}
\section{Empirical Evaluations}
In this section, we empirically validate our \textbf{Comm-FedBiO} algorithm. We consider three real-world datasets: MNIST~\cite{lecun1998gradient}, CIFAR-10~\cite{krizhevsky2009learning} and FEMNIST~\cite{caldas2018leaf}. For MNIST and CIFAR-10. We create 10 clients, for each client, we randomly sample 500 images from the original training set. For the server, we sample 500 images from the training set to construct a validation set. For FEMNIST, the entire dataset has 3,500 users and 805,263 images. We randomly select 350 users and distribute them over 10 clients. On the server side, we randomly select another 5 users to construct the validation set. Next, for label noise, we randomly perturb the labels of a portion of the samples in each client, and the portion is denoted $\rho$. We consider two settings: i.i.d. and non-i.i.d. setting. For the i.i.d. setting, all clients are perturbed with the same ratio $\rho$ and we set $\rho =0.4$ in experiments, while for the non-i.i.d. setting, each client is perturbed with a random ratio from the range of $[0.2, 0.9]$. The code is written with Pytorch, and the Federated Learning environment is simulated via Pytorch.Distributed Package. We used servers with AMD EPYC 7763 64-core CPU and 8 NVIDIA V100 GPUs to run our experiments.

\begin{figure}
\begin{center}
\includegraphics[width=0.48\columnwidth]{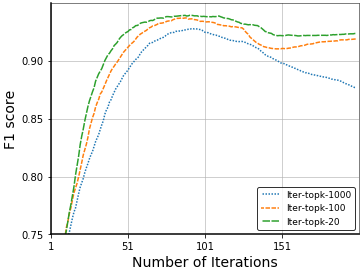}
\includegraphics[width=0.48\columnwidth]{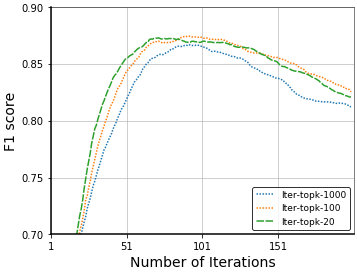}
\includegraphics[width=0.48\columnwidth]{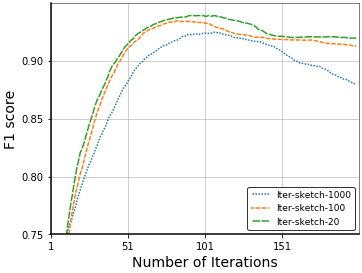}
\includegraphics[width=0.48\columnwidth]{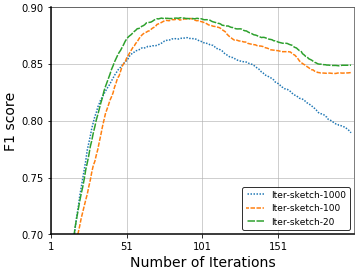}
\includegraphics[width=0.48\columnwidth]{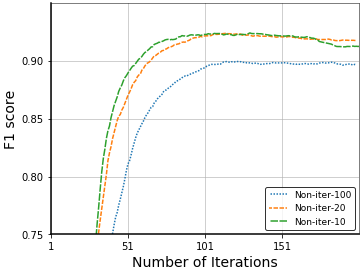}
\includegraphics[width=0.48\columnwidth]{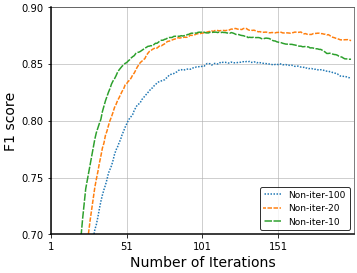}
\end{center}
\caption{F1 score at different compression rates for the MNIST data set. The plots show the results for Iter-topK, Iter-sketch, and Non-iter from top to bottom. Plots in Left show the i.i.d case and in right show the non-i.i.d case.}
\label{fig:mnist_aba}
\vspace{-0.3in}
\end{figure}

For our algorithms, we consider three variants of \textit{Comm-FedBiO} based on using different hypergradient approximation estimators: the non-iterative approximation method, the iterative approximation method with local Top-k and the iterative approximation method with Count-sketch compressor. We use Non-iter, Iter-topK and Iter-sketch as their short names. Furthermore, we also consider some baseline methods: a baseline that directly performs FedAvg~\cite{mcmahan2017communication} on the noisy dataset, an oracle method where we assume that clients know the index of clean samples (we denote this method as \textit{FedAvg-Oracle}), the \textit{FOCUS}~\cite{chen2020focus} method which reweights clients based on a 'credibility score' and the \textit{Preprocess} method~\cite{tuor2021overcoming} which uses a benchmark model to remove possibly mislabeled data before training.

We fit a model with 4 convolutional layers with 64 3$\times$3 filters for each layer. The total number of parameters is about $10^5$. We also use $L_2$ regularization with coefficient $10^{-3}$ to satisfy the strong convexity condition. Regarding hyper parameters, for three variants of our \textit{Comm-FedBiO}, we set hyper-learning rates (learning rate for sample weights) as 0.1, the learning rate as 0.01, and the local iterations $T$ as 5. We choose a minibatch of size 256 for MNIST and FEMNIST datasets and 32 for CIFAR10 datasets. For \textit{FedAvg} and \textit{FedAvg-Oracle}, we choose the learning rate, local iterations, and mini-batch size the same as in our \textit{Comm-FedBiO}. For \textit{FOCUS}~\cite{chen2020focus}, we tune its parameter $\alpha$ to report the best results, for \textit{Preprocess}~\cite{tuor2021overcoming}, we tune its parameter filtering threshold and report the best results.

We summarize the results in Figure~\ref{fig:data-clean}. Due to the existence of noisy labels, \textit{FedAvg} overfits the noisy training data quickly and the test accuracy decreases rapidly. On the contrary, our algorithm mitigates the effects of noisy labels and gets a much higher test accuracy than \textit{FedAvg}, especially for the MNIST dataset, our algorithms get a test accuracy similar to that of the oracle model. Compared to \textit{FedAvg}, our \textit{Comm-FedBiO} performs the additional hypergradient evaluation operation at each global iteration (lines 11 - 13 in Algorithm~\ref{alg:bifed-sketch-overall}). However, the additional communication overhead is negligible. In Figure~\ref{fig:data-clean}, we need the communication cost $O(d/20)$, where $d$ is the parameter dimension. Our algorithms are robust in labeling noise with almost no extra communication overhead. Finally, our algorithms also outperform the baselines \textit{FOCUS} and \textit{Preprocess}. The \textit{FOCUS} method adjusts weights at the client level, so its performance is not good when all clients have a portion of mislabeled data, As for the \textit{Preprocess} method, the benchmark model (trained over a small validation set) can screen out some mislabeled data, but its performance is sensitive to the benchmark model's performance and a 'filtering threshold' hyperparameter.

\begin{figure}
\begin{center}
\includegraphics[width=0.48\columnwidth]{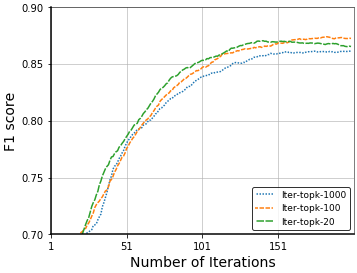}
\includegraphics[width=0.48\columnwidth]{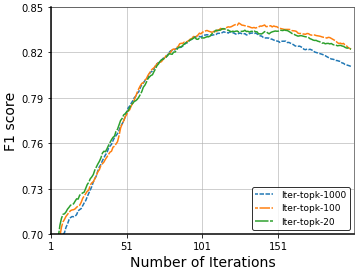}
\includegraphics[width=0.48\columnwidth]{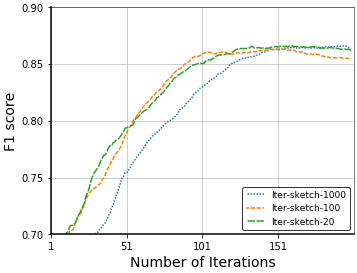}
\includegraphics[width=0.48\columnwidth]{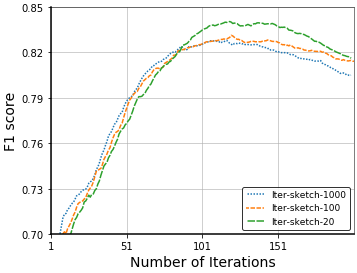}
\includegraphics[width=0.48\columnwidth]{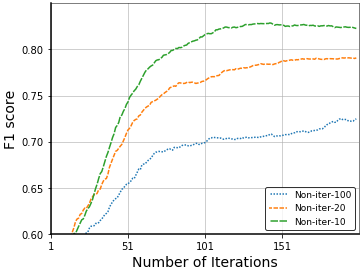}
\includegraphics[width=0.48\columnwidth]{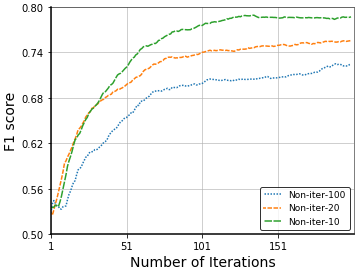}
\end{center}
\caption{F1 score at different compression rates for the FEMNIST data set. The plots show results for Iter-topK, Iter-sketch and Non-iter from top to bottom. Plots on the left show the i.i.d. case, and in the right show the non-i.i.d case.}
\label{fig:femnist_aba}
\vspace{-0.2in}
\end{figure}

Next, we verify that our algorithms are robust at different compression rates. The results are summarized in Figures~\ref{fig:mnist_aba} and~\ref{fig:femnist_aba}. We use the $F1$ score to measure the efficacy of our algorithms in identifying mislabeled samples. We use 0.5 as the threshold for mislabeled data: for all samples with weights smaller than 0.5, we assume that they are mislabeled. Then the F1 score is computed between the ground-truth mislabeled samples and predicted mislabeled samples of our algorithms. In Figures~\ref{fig:mnist_aba} and~\ref{fig:femnist_aba}, we show the results of i.i.d. and non-i.i.d. cases for the MNIST and FEMNIST datasets. The iterative algorithms (Iter-topK and Iter-sketch) achieve higher compression rates than the Non-iter method. For iterative algorithms, performance decreases at the compression rate 1000$\times$, while the Non-iter method works well at around 10$\times$ to 100$\times$. A major reason for this phenomenon is that the gradient $\nabla q(v)$ is highly sparse in experiments, whereas the Hessian matrix itself is much denser.

\section{Conclusion}
In this paper, we study the Federated Learning problem with noisy labels. We propose to use Shapley Value as a measure of the sample contribution. As Shapley Value is intractable, we then propose a Federated Bilevel Optimization formulation as its alternative. Next, we propose \textit{Comm-FedBiO} to solve the Federated Bilevel Optimization problem, more specifically, we introduce two subroutines to estimate the hypergradient \emph{i.e.} the Iterative and Non-iterative algorithms. We provide a theoretical convergence guarantee for both methods. In experiments, we validate our algorithms using real-world datasets. All empirical results show a superior performance of our proposed methods on various baselines.

\bibliographystyle{abbrv}
\bibliography{sample-authordraft.bib}

\appendix
\section{Proof for iterative algorithm}
In the iterative algorithm, we optimize Eq.~\eqref{eq:approx-lr2}. The full version of Theorem~\ref{theo:sketch2} in the main text is stated as follows:
\begin{theorem}(Theorem~\ref{theo:sketch2})
Under Assumption~\ref{basic_assumption},~\ref{sketch2_assumption} and $||v^i|| \le D_v$, $\alpha = \frac{8}{\mu_{G}(i+a)}$, with $a>\max(1, \frac{2-\tau}{\tau}(\sqrt{\frac{2}{2-\tau}}+ 1))$ being some shift constant. Then, if we use the count-sketch size $O(log(dI/\delta)/\tau)$, with probability $1 - \delta$, we have the following:
\begin{equation*}
  E[||v^I - v^*||^2] \le \frac{C_1}{I^3} +\frac{C_2}{I^2} +  \frac{C_3}{I^2}
\end{equation*}
where $G_q^2 = 2L_{G}^2D_v^2 + 2C_{F}^2$, $C_1 = 3a^3D_v^2/4$, $C_2 = (384(2L_{G} + \mu_{G})G_q^2)/(\mu_G^3\tau(1 - (1 -\frac{\tau}{2})(1 + \frac{1}{a})^2))$, $C_3 = 12(I+2a)G_q^2/\mu_{G}^2$
\end{theorem}

\begin{proof}
We first show that $E||\nabla q(v^{i})||^2$ is upper bounded: $||\nabla q(v^{i})||^2 =  ||\nabla_{yy}^2 G(x,y_x)v^i - \nabla_y F(x, y_x))||^2 \le 2L_{G}^2D_v^2 + 2C_{F}^2$. We denote $G_q^2 = 2L_{G}^2D_v^2 + 2C_{F}^2$. Next, following the analysis in~\cite{stich2018sparsified, karimireddy2019error}, we consider the virtual sequence $\tilde{v}^{i} = v^{i} - e^{i}$, where we have:
\begin{equation*}
\begin{split}
     \tilde{v}^{i} &= v^{i} - e^{i} = v^{i} - \alpha_{i-1} \nabla q(v^{i-1}) - e^{i-1} + C(\alpha_{i-1} \nabla q(v^{i-1}) + e^{i-1})\\ &= v^{i-1} - e^{i-1} - \alpha_{i-1} \nabla q(v^{i-1}) = \tilde{v}^{i-1} - \alpha_{i-1} \nabla q(v^{i-1})
\end{split}
\end{equation*}
Then we have:
\begin{equation}
\label{eq:bound1}
    \begin{split}
        &||\tilde{v}^{i} - v^*||^2 = ||\tilde{v}^{i-1} - \alpha_{i-1} \nabla q(v^{i-1}) - v^*||^2\\ 
        =& ||\tilde{v}^{i-1} - v^*||^2 - 2\alpha_{i-1}\langle \tilde{v}^{i-1} - v^*, \nabla q(v^{i-1})\rangle + \alpha_{i-1}^2||\nabla q(v^{i-1})||^2\\
        \le& ||\tilde{v}^{i-1} - v^*||^2 - 2\alpha_{i-1}\langle \tilde{v}^{i-1} - v^{i-1}, \nabla q(v^{i-1})\rangle \\
        & + 2\alpha_{i-1}\langle v^* - v^{i-1}, \nabla q(v^{i-1})\rangle + \alpha_{i-1}^2G_q^2
    \end{split}
\end{equation}
In the last inequality, we use the fact that $\nabla q(v^{i-1})$ is upper-bounded. Then, since $q(v)$ is strongly convex, we have $q(v^*) \ge q(v^{i-1}) + \langle v^* - v^{i-1}, \nabla q(v^{i-1}) \rangle + \frac{\mu_{G}}{2}||v^* - v^{i-1}||^2$. Furthermore, by the triangle inequality, we have: $||v^* - v^{i-1}||^2 \ge \frac{1}{2}||v^* - \tilde{v}^{i-1}||^2 - ||\tilde{v}^{i-1} - v^{i-1}||^2$. Combine these two inequalities, we can upper bound the third term in Eq.~(\ref{eq:bound1}) and have:
\begin{equation}
\label{eq:bound2}
    \begin{split}
        ||\tilde{v}^{i} - v^*||^2 \le& (1 - \frac{\mu_{G}\alpha_{i-1}}{2})||\tilde{v}^{i-1} - v^*||^2 - 2\alpha_{i-1}\langle e^{i-1}, \nabla q(v^{i-1})\rangle\\
        &+ \mu_{G}\alpha_{i-1} ||e^{i-1}||^2 - 2\alpha_{i-1}(q(v^{i-1}) - q(v^*)) + \alpha_{i-1}^2G_q^2
    \end{split}
\end{equation}
Now, we bound $\langle e^{i-1}, \nabla q(v^{i-1})\rangle$, first by the triangle inequality, we have: $-\langle e^{i-1}, \nabla q(v^{i-1})\rangle \le||e^{i-1}||||\nabla q(v^{i-1})||\le(L_{G_{y}}||e^{i-1}||^2 + \frac{1}{4L_{G_{y}}}||\nabla q(v^{i-1})||^2)$. Then, by the smoothness of $q(v)$, we have
\[-\langle e^{i-1}, \nabla q(v^{i-1})\rangle \le(L_{G_{y}}||e^{i-1}||^2 + \frac{1}{2}(q(v^{i-1}) - q(v^*)))\] combine the two inequalities with Eq.~(\ref{eq:bound2}), we have:
\begin{equation}
\label{eq:bound3}
\begin{split}
    q(v^{i-1}) - q(v^*) \le& \frac{(1 - \mu_{G}\alpha_{i-1}/2)}{\alpha_{i-1}} ||\tilde{v}^{i-1} - v^*||^2 - \frac{1}{\alpha_{i-1}} ||\tilde{v}^{i} - v^*||^2\\
    &+ (2L_{G_{y}} + \mu_{G})||e^{i-1}||^2 + \alpha_{i-1} G_q^2
\end{split}
\end{equation}
Note that we rearrange the terms and move $q(v^{i-1}) - q(v^*)$ to the left.
Now, we bound the term $||e^{i}||^2$. By Assumption~\ref{sketch2_assumption}, and the count sketch memory complexity in~\cite{charikar2002finding}, suppose that we use the count sketch compressor and the compressed gradients have dimension $O(log(d/\delta)/\tau)$, we can recover the $\tau$ heavy hitters with probability at least $1 -\delta$. Since we transfer compressed gradients $I$ times, we have the communication cost of $O(log(dI/\delta)/\tau)$ by a union bound. Then for all $i \in [I]$, we have:
\begin{equation}
\label{eq:e_bound}
\begin{split}
    ||e^i||^2 =& ||\alpha_{i-1} \nabla q(v^{i-1}) + e^{i-1} - C(\alpha_{i-1} \nabla q(v^{i-1}) + e^{i-1})||^2\\
    \le& (1 -\tau)||\alpha_{i-1} \nabla q(v^{i-1}) + e^{i-1}||^2\\
    \le& (1 -\tau)(\alpha_{i-1}^2(1 + \frac{1}{\gamma})||\nabla q(v^{i-1})||^2 + (1 + \gamma)||e^{i-1}||^2)\\
    \le& (1 -\tau)(1 + \gamma)||e^{i-1}||^2 + (1 -\tau)\alpha_{i-1}^2(1 + \frac{1}{\gamma})G_q^2\\
\end{split}
\end{equation}
Next, we choose $\gamma = \frac{\tau}{2(1 - \tau)}$, then we can prove \[||e^i||^2 \le \frac{(1 -\tau)(2-\tau)(1 + \frac{1}{a})^2\alpha_i^2G_q^2}{\tau(1 - (1 -\frac{\tau}{2})(1 + \frac{1}{a})^2)}\] by induction, we omit the derivation here due to space limitation. 
By $a > \frac{2-\tau}{\tau}(\sqrt{\frac{2}{2-\tau}}+ 1)$, so the denominator is positive. Inserting the bound for $||e^i||^2$ back to Eq.~(\ref{eq:bound3}), we have:
\begin{equation*}
\begin{split}
    q(v^{i-1}) - q(v^*) \le& \frac{(1 - \frac{\mu_{G}\alpha_{i-1}}{2})}{\alpha_{i-1}} ||\tilde{v}^{i-1} - v^*||^2 - \frac{1}{\alpha_{i-1}} ||\tilde{v}^{i} - v^*||^2 \\
    & +\frac{2(2L_{G_{y}}
    + \mu_{G})\alpha_{i-1}^2G_q^2}{\tau(1 - (1 -\frac{\tau}{2})(1 + \frac{1}{a})^2)} + \alpha_{i-1} G_q^2
\end{split}
\end{equation*}
Finally, we average $v^i$ with weight $w_i = (i + a)^2$, choose $\alpha = \frac{8}{\mu_{G}(i+a)}$, then by Lemma 3.3 in~\cite{stich2018sparsified} and the strong convexity of $q(v)$, we get the upper bound of $||v^I - v^*||^2$ as shown in the Theorem.
\end{proof}

\section{Proof for Non-iterative algorithm}
The proof for Corollary~\ref{cor:sub-embed} is included in Theorem~9 in~\cite{woodruff2014sketching}. The full version of Theorem~\ref{theo:sketch1} in the main text is as follows:
\begin{theorem}(Theorem~\ref{theo:sketch1})
For any given $\epsilon, \delta \in (0, 1/2)$, if $S_1 \in \mathbb{R}^{r_1 \times d}$ is a $(\lambda_1\epsilon, \delta/2)$ sketch matrix and $S_{2} \in \mathbb{R}^{r_2 \times d}$ is a $(\lambda_2\epsilon, \delta/2)$ sketch matrix. Under Assumption~\ref{sketch1_assumption}, with probability at least $1 - \delta$ we have:
\begin{equation*}
\begin{split}
||\hat{\nabla} h(x) - \nabla h(x)|| \le \epsilon ||v^{*}||
\end{split}
\end{equation*}
where
$\lambda_1 = \frac{5\mu_{G}}{7\sqrt{r_s}C_{G_{xy}}L_{G_y}}, \lambda_2 = \frac{1}{3(r_1+1)}$.
\end{theorem}

\begin{proof}
For convenience, we denote $H_{yy} = \nabla_{yy}^2 G(x,y_x)$, $H_{xy} = \nabla_{xy}^2 G(x,y_x)$, $g_y = \nabla_y F(x, y_x)$, $g_x = \nabla_x F(x, y_x)$, $g = \nabla h(x)$, $\hat{g} = \hat{\nabla} h(x)$, and denote $\epsilon_1 = \lambda_1 \epsilon$, $\epsilon_2 = (r_1 + 1)\lambda_2\epsilon$. Furthermore, we denote $v^{*} = \underset{v}{\arg\min} || H_{yy}  v - g_y||_2^2$, $\hat{\omega} = \underset{\omega}{\arg\min} ||S_{2}  H_{yy}  S_1^T  \omega - S_{2}  g_y||_2^2$, $v_{s_1} = \underset{v}{\arg\min} ||H_{yy}  S_1^T  S_1  v - g_y||_2^2$ and $\omega_{s_1} = \underset{\omega}{\arg\min} ||H_{yy}  S_1^T  \omega - g_y||_2^2$. Then we have the following.
\begin{equation}
\label{eq:s2}
\begin{split}
    &(1- \epsilon_2)|| H_{yy}  S_1^T  \hat{\omega} - g_y|| \le ||S_{2}  H_{yy}  S_1^T  \hat{\omega} - S_{2}  g_y||\\
    &\le  ||S_{2}  H_{yy}  S_1^T  \omega_{s_1} - S_{2}  g_y||
    \le (1+ \epsilon_2)|| H_{yy}  S_1^T  \omega_{s_1} - g_y||
\end{split}
\end{equation}
The second inequality is by the definition of $\hat{\omega}$, the first and third inequality use the fact that $S_2$ is a ($\lambda_2\epsilon$, $\delta/2$) sketch matrix and by Corollary~\ref{cor:sub-embed}, it is a $\lambda_2\epsilon\times (r_1 + 1) = \epsilon_2$ subspace embedding matrix over the column space $[H_{yy}  S_1^T, g_y]$, so Eq.~(\ref{eq:s2}) holds with probability $1 - \delta/2$. Next, since $S_1$ is a ($\epsilon_1$, $\delta/2$) sketching matrix, with probability $1 - \delta/2$, we have:
\begin{equation}
\label{eq:s1}
    \begin{split}
        &||H_{yy}  S_1^T  S_1  v^{*} - H_{yy}  v^{*}|| \le \epsilon_1||v^{*}||_F||H_{yy}||_F \\
        &\rightarrow||(H_{yy}  S_1^T  S_1  v^{*} - g_y) - (H_{yy}  v^{*} - g_y)|| \le \epsilon_1 ||v^{*}||_F||H_{yy}||_F\\
         &\rightarrow||H_{yy}  S_1^T  S_1  v^{*} - g_y|| \le ||H_{yy}  v^{*} - g_y|| + \epsilon_1 ||v^{*}||_F||H_{yy}||_F
    \end{split}
\end{equation}
In the last step, we use the triangle inequality $||x|| - ||y|| \le ||x - y||$. Combining Eq.~(\ref{eq:s1}) and the definition of $v_{s_1}$, also noticing that $span(S_1  v) \subset span(\omega)$, we have:
\begin{equation}
\label{eq:def}
    \begin{split}
        &||H_{yy}  S_1^T   \omega_{s_1} - g_y|| \le ||H_{yy}  S_1^T  S_1  v_{s_1} - g_y||\\
        &\le ||H_{yy}  S_1^T  S_1  v^{*} - g_y|| \le  ||H_{yy}  v^{*} - g_y|| + \epsilon_1 ||v^{*}||_F||H_{yy}||_F
    \end{split}
\end{equation}
Finally we combine Eq.~(\ref{eq:s2}) and (\ref{eq:def}) to get:
\begin{equation}
\label{eq:bound}
    \begin{split}
        ||H_{yy}  S_1^T  \hat{\omega} - g_y|| \le \frac{(1+ \epsilon_2)}{(1- \epsilon_2)}(||H_{yy}  v^{*} - g_y|| + \epsilon_1 ||v^{*}||_F||H_{yy}||_F)
    \end{split}
\end{equation}
By the union bound, Eq.~(\ref{eq:bound}) holds with probability $1 - \delta$. Since $H_{yy}$ is positive definite (invertible), we have $||H_{yy}  v^{*} - g_y|| = 0$. Eq.~(\ref{eq:bound}) can be simplified further as:
\[
    ||H_{yy}  S_1^T  \hat{\omega} - g_y|| \le \frac{\epsilon_1(1+ \epsilon_2)}{(1- \epsilon_2)} ||v^{*}||_F||H_{yy}||_F
\]
Moreover, $G(x,y)$ is $\mu_G$-strongly convex (Assumption~\ref{basic_assumption},), we have:
\begin{equation}
    \begin{split}
        &||H_{yy}  S_1^T  \hat{\omega} - g_y||^2 = ||H_{yy}  S_1^T  \hat{\omega} - g_y - (H_{yy}  v^{*} - g_y)||^2\\ =& ||H_{yy}  (S_1^T\hat{\omega} - v^*)||^2 \ge \mu_{G}^2||S_1^T  \hat{\omega} - v^{*}||^2
    \end{split}
\end{equation}
Combining the above two equations, we have the following.
\begin{equation}
\label{eq:final-bound}
    \begin{split}
        ||S_1^T  \hat{\omega} - v^{*}|| \le \frac{\epsilon_1(1+ \epsilon_2)}{\mu_{G}(1- \epsilon_2)} ||v^{*}||_F||H_{yy}||_F
    \end{split}
\end{equation}
As for $||H_{yy}||_F$, by Assumption C, we have $||H_{yy}||_F = \sqrt{\sum_{i} \sigma_i^2} \le \sqrt{r_s}\sigma_{max} = \sqrt{r_s}L_{G_y}$. Then Eq.~(\ref{eq:final-bound}) can be simplified to $||S_1^T  \hat{\omega} - v^{*}|| \le \frac{\sqrt{r_s}L_{G_y}\epsilon_1(1+ \epsilon_2)}{\mu_{G}(1- \epsilon_2)} ||v^{*}||$. Then we have the following:
\begin{equation}
    \begin{split}
        ||\hat{g} - g|| &= ||H_{xy}  S_1^T  \hat{\omega} - H_{xy}  v^{*}|| = ||H_{xy}  (S_1^T  \hat{\omega} - v^{*})|| \\
        &\le \frac{\sqrt{r_s}C_{G_{xy}}L_{G_y}\epsilon_1(1+ \epsilon_2)}{\mu_{G}(1- \epsilon_2)} ||v^{*}||
    \end{split}
\end{equation}
By choice of $\epsilon_1 = \frac{5\mu_{G}\epsilon}{7\sqrt{r_s}C_{G_{yx}}L_{G_y}}$ and $\epsilon_2 = \frac{\epsilon}{3} < \frac{1}{6}$, \emph{i.e.} $\frac{1 + \epsilon_2}{1 - \epsilon_2} = 1 + \frac{2\epsilon_2}{1 - \epsilon_2} < \frac{7}{5}$.
So, we get $||\hat{g} - g|| < \epsilon||v^{*}||$ with probability at least $1 - \delta$. The proof is complete.
\end{proof}

\section{Proof for Convergence Analysis}
The full version of Theorem~\ref{theo:overall} in the main text is provided as follows:
\begin{theorem}(Theorem~\ref{theo:overall})
Under Assumption~\ref{basic_assumption},~\ref{higher-order_assumption}, we pick the learning rate $\eta = \frac{1}{2L_h\sqrt{K+1}}$, then we have:
\begin{itemize}
\item[a)] Suppose that $\{x_k\}_{k\ge0}$ is generated from the non-iterative Algorithm~3, under Assumption~\ref{sketch1_assumption} and $\epsilon = (K+1)^{-1/4}$, we have the following:
\begin{equation*}
    E[||\nabla h(x_k)||^2] \le \left(32L_h\left(h(x_0) - h(x^{*})\right) + \frac{16C_{F}^2}{\mu_{G}^2}\right)\frac{1}{\sqrt{K}} 
\end{equation*}
\item[b)] Suppose $\{x_k\}_{k\ge0}$ are generated from the non-iterative Algorithm~4, under Assumption~\ref{sketch2_assumption}, $I = L_h\sqrt{K+1}$, we have:
\begin{equation*}\begin{split}
    E[||\nabla h(x_k)||^2] \le&  \frac{C_1}{K^{3/2}} + \frac{C_{2}}{K} + \frac{C_3}{\sqrt{K}}
\end{split}
\end{equation*}
where $C_1$, $C_{2}$ and $C_3$ are some constants. $C_1 = 12C_{G_{xy}}^2a^3D_v^2/L_h^3$, $C_{2} = (6144(2L_{G_{y}} + \mu_{G})C_{G_{xy}}^2G_q^2)/(\mu_G^3\tau(1 - (1 -\frac{\tau}{2})(1 + \frac{1}{a})^2)L_h^2) + 384aC_{G_{xy}}^2G_q^2/\mu_{G}^2L_h^2$, $C_3 = 32L_h(h(x_0) - h(x^{*})) + (192C_{G_{xy}}^2G_q^2)/\mu_{G}^2L_h$.
\end{itemize}
\end{theorem}

\begin{proof}
As stated in Proposition~\ref{lemma:smooth}, $h(x)$ is $L_h$ smooth: $h(x_{k+1}) \le h(x_k) + \langle \nabla h(x_k), x_{k+1} - x_k \rangle + \frac{L_h}{2} ||x_{k+1} - x_k||^2$, and by the update rule of outer variable $x_{k+1} = x_k - \eta \hat{\nabla} h(x_k)$, we have:
\begin{equation*}
    \begin{split}
    h(x_{k+1}) \le& h(x_k) - \eta \langle \nabla h(x_k), \hat{\nabla} h(x_k)\rangle + \frac{L_h}{2} \eta^2 ||\hat{\nabla} h(x_k)||^2\\
    \le& h(x_k) - \eta ||\nabla h(x_k)||^2 + \eta \langle \nabla h(x_k),  \nabla h(x_k) - \hat{\nabla} h(x_k)\rangle\\ + &\frac{L_h}{2} \eta^2||\hat{\nabla} h(x_k) - \nabla h(x_k) + \nabla h(x_k)||^2\\
    \end{split}
\end{equation*}
Using the triangle inequality and the Cauchy-Schwarz inequality, we have $\langle \nabla h(x_k), \nabla h(x_k) - \hat{\nabla} h(x_k)\rangle \le \frac{1}{2}||\nabla h(x_k)||^2 + \frac{1}{2} ||[\hat{\nabla} h(x_k) - \nabla h(x_k)||^2$
and $||\hat{\nabla} h(x_k) - \nabla h(x_k) + \nabla h(x_k)||^2 \le 2||\hat{\nabla} h(x_k) - \nabla h(x_k)||^2 + 2||\nabla h(x_k)||^2$. We combine these two inequalities with the above inequality.
\begin{equation}\begin{split}\label{eq:conv_bound2}
    h(x_{k+1}) \le& h(x_k) - \eta(\frac{1}{2} - \eta L_h) ||\nabla h(x_k)||^2 + \eta (\frac{1}{2} + \eta L_h)e_k\\
\end{split}\end{equation}
where we denote $||\hat{\nabla} h(x_k) - \nabla h(x_k)||^2$ as $e_k$. Next we prove the two cases respectively. 

Case~(a): By Theorem~\ref{theo:sketch1} and $||v^*|| = ||\nabla_{yy}^2 G(x,y_x)^{-1}\nabla_y F(x, y_x)|| \le \frac{ C_{F}}{\mu_{G}}$, we have:
$e_k \le \epsilon^2 ||v_k^{*}||^2 \le \epsilon^2 C_{F_y}^2/\mu_{G}^2$. Combine this inequality with Eq.~\eqref{eq:conv_bound2} and telescope from $1$ to $K$, we have::
\begin{equation*}
    \begin{split}
        \sum_{k=0}^{K-1} \eta(\frac{1}{2} - \eta L_h)||\nabla h(x_k)||^2 \le& h(x_0) - h(x^{*}) + \sum_{k=0}^{K-1} \eta(\frac{1}{2}+ \eta L_h)\frac{\epsilon^2 C_{F}^2}{\mu_{G}^2}
    \end{split}
\end{equation*}
We select $x_k$ with probability proportional to $\eta(\frac{1}{2} - \eta L_h)$, and pick $\eta = \frac{1}{2L_h\sqrt{K+1}}$, $\epsilon = (K+1)^{-1/4}$ we have:
\begin{equation*}
    E[||\nabla h(x_k)||^2] \le \frac{1}{\sqrt{K}} \bigg(32L_h(h(x_0) - h(x^{*})) + \frac{16C_{F}^2}{\mu_{G}^2}\bigg)
\end{equation*}
where we use $\sum_{k=0}^{K-1} \eta(\frac{1}{2} - \eta L_h) \ge \frac{\sqrt{K}}{32L_h}$, $\sum_{k=0}^{K-1} \epsilon^2\eta(\frac{1}{2} + \eta L_h) \le \frac{1}{2L_h}$.

Case~(b): for the iterative algorithm, we notice that  $e_k = ||\nabla_x F(x, y_x) -\nabla_{xy}^2 G(x,y_x)v^I - \nabla_x F(x, y_x) -\nabla_{xy}^2 G(x,y_x)  v^{*}||^2 \le C_{G_{xy}}^2||(v^I - v^{*})||^2$. Combine this inequality with Eq.~(\ref{eq:conv_bound2}) and telescope from $1$ to $K$ as case (a), we have:
\begin{equation*}
    \begin{split}
        &\sum_{k=0}^{K-1} \eta(\frac{1}{2} - \eta L_h)||\nabla h(x_k)||^2 \\
        \le& h(x_0) - h(x^{*}) + \sum_{k=0}^{K-1} \eta(\frac{1}{2}+ \eta L_h)C_{G_{xy}}^2||(v^I - v^{*})||^2
    \end{split}
\end{equation*}
By choosing the learning rate $eta$ and $I$ as in the Theorem, then combine Theorem~\ref{theo:sketch2}, it is straightforward to get the upper bound of $||\nabla h(x_k)||^2$ as stated in the Theorem.
\end{proof}

\end{document}